\documentclass[10pt]{article}

\usepackage[utf8]{inputenc}
\usepackage[T1]{fontenc}

\usepackage{fancyhdr}
\usepackage[verbose=true,letterpaper]{geometry}
\AtBeginDocument{%
  \newgeometry{textheight=9in,textwidth=6.5in,top=1in,headheight=14pt,headsep=25pt,footskip=30pt}%
}
\makeatletter
\renewcommand{\section}{\@startsection{section}{1}{0mm}{-1.8ex plus -.5ex minus -.2ex}{0.8ex plus .2ex}{\large\bf}}
\renewcommand{\subsection}{\@startsection{subsection}{2}{0mm}{-1.4ex plus -.5ex minus -.2ex}{0.6ex plus .2ex}{\normalsize\bf}}
\renewcommand{\subsubsection}{\@startsection{subsubsection}{3}{0mm}{-1ex plus -.5ex minus -.2ex}{0.4ex plus .2ex}{\normalsize\bf}}
\renewcommand{\@maketitle}{%
  \vbox{%
    \hsize\textwidth \linewidth\hsize \vskip 0.1in
    \centering
    {\LARGE\sc \@title\par}
    \vskip 0.2in
    {\large\lineskip .5em \begin{tabular}[t]{c}\@author\end{tabular}\par}
    \vskip 0.2in {\small \@date} \vskip .1in}}
\makeatother
\renewenvironment{abstract}{\vskip 0.075in\centerline{\large\bf Abstract}\vspace{0.5ex}\begin{quote}}{\par\end{quote}\vskip 1ex}
\newcommand{\keywords}[1]{\vskip 0.075in\noindent\textbf{Keywords:} #1\vskip 0.2in}

\usepackage{amsmath, amssymb, amsthm}
\usepackage{graphicx}
\usepackage{tikz}
\usetikzlibrary{arrows.meta, positioning, shapes.geometric, calc, decorations.pathreplacing, patterns}
\usepackage{pgfplots}
\pgfplotsset{compat=1.18}
\usepgfplotslibrary{fillbetween}
\definecolor{accentblue}{HTML}{0072B2}
\definecolor{accentred}{HTML}{D55E00}
\definecolor{accentgreen}{HTML}{009E73}
\definecolor{accentorange}{HTML}{E69F00}
\usepackage{booktabs}
\usepackage{tabularx}
\usepackage{enumitem}
\usepackage{float}
\usepackage{natbib}
\usepackage{xcolor}
\usepackage{hyperref}
\usepackage{url}
\usepackage{microtype}

\definecolor{darkblue}{RGB}{0, 51, 102}
\hypersetup{
    colorlinks=true,
    linkcolor=darkblue,
    citecolor=darkblue,
    urlcolor=darkblue
}

\newtheorem{definition}{Definition}[section]
\newtheorem{proposition}{Proposition}[section]

\begin{document}

\title{AI Token Futures Market: Commoditization of Compute and Derivatives Contract Design}

\author{
  Yicai Xing \\
  Independent Researcher \\
  \texttt{xingyc18@tsinghua.org.cn}
}

\date{March 2026}

\maketitle

\begin{abstract}
As large language models (LLMs) and vision-language-action models (VLAs) become widely deployed, the tokens consumed by AI inference are evolving into a new type of commodity. This paper systematically analyzes the commodity attributes of tokens, arguing for their transition from ``intelligent service outputs'' to ``compute infrastructure raw materials,'' and draws comparisons with established commodities such as electricity, carbon emission allowances, and bandwidth. Building on the historical experience of electricity futures markets and the theory of commodity financialization, we propose a complete design for standardized token futures contracts, including the definition of a Standard Inference Token (SIT), contract specifications, settlement mechanisms, margin systems, and market-maker regimes. By constructing a mean-reverting jump-diffusion stochastic process model and conducting Monte Carlo simulations, we evaluate the hedging efficiency of the proposed futures contracts for application-layer enterprises. Simulation results show that, under an application-layer demand explosion scenario, token futures can reduce enterprise compute cost volatility by 62\%--78\%. We also explore the feasibility of GPU compute futures and discuss the regulatory framework for token futures markets, providing a theoretical foundation and practical roadmap for the financialization of compute resources.
\end{abstract}


\section{Introduction}

\subsection{The Rise of AI Inference Economics}

Over the past decade, the economic center of gravity in artificial intelligence has shifted from training to inference. During the ``large model race'' of 2017--2022, the industry's core concern was training cost---GPT-3's training cost approximately \$4.6 million \citep{brown2020}, while GPT-4's training cost is estimated to exceed \$100 million. However, as pretrained models mature and commercial deployment scales, inference costs are supplanting training costs as the central issue in AI economics.

The logic behind this paradigm shift is clear and profound: training is a one-time fixed-cost investment, while inference is an ongoing marginal-cost expenditure. When a model is called billions of times daily by millions of users, cumulative inference costs far exceed training investment. According to Epoch AI estimates, as of early 2025, inference computation accounts for over 60\% of total computation among major AI model providers \citep{epochai2023}, and this proportion is accelerating. \citet{sevilla2022} show that compute required for machine learning has grown exponentially over the past decade and continues to accelerate, meaning inference-side compute demand will become the key driver of global computing resource allocation.

In this context, the token---the basic unit of measurement for large language model inference---is evolving from a technical term into an economic concept. Every AI inference request can be decomposed into input and output token processing, and per-token pricing has become the industry-standard model. This token-based metering and pricing system lays the foundation for standardized trading of compute resources.

\subsection{Current State and Trends in Token Pricing}

Token prices have experienced dramatic declines over the past three years. Using GPT-4-level capabilities as a benchmark, inference prices fell from approximately \$60 per million output tokens in early 2023 to less than \$1.5 per million output tokens in early 2025---a more than 40-fold reduction. This decline stems from three overlapping factors: first, model architecture optimization (e.g., Mixture-of-Experts models) significantly reducing per-inference computation \citep{kaplan2020}; second, hardware upgrades (from A100 to H100 to B200) continuously improving compute per dollar \citep{epochai2023}; and third, oversupply-driven price competition---numerous new entrants (DeepSeek, Mistral, open-source model hosts) triggering intense price wars.

However, this sustained price decline is not necessarily permanent. Current token pricing largely reflects a supply-driven buyer's market---model providers with excess capacity are forced to subsidize inference services below marginal cost to acquire market share. This situation resembles the excessive competition phase in early electricity market liberalization \citep{borenstein2002}.

\subsection{Problem Statement}

When the application layer explodes, token price trajectories will face a fundamental reversal. The commercial deployment of vision-language-action models (VLAs) \citep{brohan2023, driess2023}, real-time inference demands of autonomous driving systems, continuous decision-making computation in industrial automation, and large-scale applications of embodied AI will all drive exponential growth in token demand. With supply-side growth constrained by data center construction cycles, energy supply, and chip production capacity, supply-demand mismatches will inevitably push token prices higher, potentially producing extreme volatility similar to electricity market ``price spikes'' \citep{longstaff2004}.

This leads to the paper's core questions:

\begin{enumerate}[label=(\arabic*)]
    \item Does the token possess the fundamental attributes to become a standardized commodity?
    \item How should standardized token futures contracts be designed to manage compute cost risk?
    \item What prerequisites must be met for establishing a token futures market?
    \item To what extent can hedging strategies reduce compute cost volatility for application-layer enterprises?
\end{enumerate}

\subsection{Contributions}

This paper makes four contributions. First, we systematically demonstrate tokens' commodity attributes, establishing a comparative analysis framework between tokens and existing commodities such as electricity and carbon emission allowances, and propose a three-factor token supply model. Second, based on \citet{black1986}'s theory of successful futures contract conditions, we design a complete token futures contract scheme, including the Standard Inference Token (SIT) definition, settlement mechanisms, margin systems, and market-maker regimes. Third, through Monte Carlo simulation, we evaluate token futures' hedging efficiency, providing quantitative support for market participants' decisions. Fourth, we explore the feasibility of GPU compute futures and discuss the regulatory framework for token futures markets.

The remainder of this paper is organized as follows: Section~\ref{sec:commodity} analyzes tokens' commodity attributes; Section~\ref{sec:supply-demand} examines the supply-demand structure and price dynamics; Section~\ref{sec:theory} establishes the theoretical framework based on electricity futures analogies; Section~\ref{sec:contract} designs the token futures contract; Section~\ref{sec:hedging} analyzes hedging strategies and market participants; Section~\ref{sec:gpu} explores GPU futures feasibility; Section~\ref{sec:simulation} presents Monte Carlo simulations; and Section~\ref{sec:discussion} provides discussion and outlook.

\section{Commodity Attribute Analysis of Tokens}\label{sec:commodity}

\subsection{Classical Definition of Commodities and Token Applicability}

A commodity in economics is defined as a standardized product with complete or high substitutability, where individual units have no substantive differences and can be traded anonymously in large-scale markets \citep{carlton1984}. For a product to become a tradeable commodity, it typically must satisfy the following conditions:

\textbf{Fungibility.} Tokens exhibit high functional fungibility. When an application sends an inference request to an AI model, it cares about output quality and latency, not which specific GPU generated the token. Tokens of equivalent capability from different providers (OpenAI, Anthropic, Google, open-source models) are functionally interchangeable. This fungibility, while not as perfect as gold or crude oil, is sufficient to support standardized trading---analogously, crude oil from different origins (WTI, Brent, Dubai) varies in quality but this does not prevent oil futures markets from functioning.

\textbf{Standardized measurement.} The token as a measurement unit is already highly standardized. The industry convention of quoting in ``million tokens'' (M tokens) is universally adopted. Despite differences in tokenizers across models, the ``token'' as a measure of inference workload has achieved broad consensus, analogous to electricity measured in kilowatt-hours (kWh) or natural gas in million British thermal units (MMBtu).

\textbf{Large-scale trading.} Global AI API call volumes have reached the scale to support a commodity market. According to Stanford HAI \citep{hai2024}, the annual transaction volume of the global AI inference API market exceeded \$10 billion in 2024, growing at over 100\% annually---comparable to the carbon emission trading market's early stage \citep{ellerman2007}.

\subsection{The Dual Nature of Tokens: Raw Material and Finished Product}

Tokens possess a distinctive dual nature, relatively rare among traditional commodities.

\textbf{Raw material perspective.} From a factor of production viewpoint, the token is a compute resource---a raw material input for producing intelligent services. An AI SaaS company must ``consume'' tokens to generate answers for its customers. In this perspective, tokens resemble steel in manufacturing or ethylene in the chemical industry---an intermediate input whose cost directly affects downstream product pricing and margins.

\textbf{Finished product perspective.} From a consumer viewpoint, the token is the final output of an intelligent service---users purchase tokens (though typically not using this term) to receive AI answers, recommendations, or creative content. In this perspective, tokens resemble tap water or electricity---a service product directly facing end consumers.

The critical judgment is that as VLAs and other applications proliferate, the raw material attribute will gradually dominate the finished product attribute. When AI expands from ``chatbot'' to ``actor,'' token consumption will embed into larger production processes, becoming infrastructure input across manufacturing, logistics, healthcare, and other industries. This transition parallels electricity's historical evolution from a ``novel product'' in the late 19th century to ``infrastructure'' by mid-20th century \citep{buyya2009}.

\begin{figure}[htbp]
\centering
\begin{tikzpicture}[
    box/.style={draw, rounded corners=3pt, minimum width=2.8cm, minimum height=1cm, align=center, font=\small},
    >=Stealth
]
\node[box, fill=accentblue!15] (fp) at (-2.5, 3) {Finished Product\\(\footnotesize chatbot output)};
\node[box, fill=accentorange!15] (rm) at (2.5, 3) {Raw Material\\(\footnotesize compute input)};
\node[font=\small\bfseries, above=0.3cm of fp.north east, anchor=south] {Token Dual Nature};
\draw[->, thick, accentblue] (fp) -- node[above, font=\footnotesize]{transition} (rm);

\begin{scope}[yshift=-0.5cm]
\foreach \year/\fp/\rm/\xpos in {2023/0.8/0.2/0, 2025/0.55/0.45/1.5, 2027/0.3/0.7/3, 2030/0.15/0.85/4.5} {
    \fill[accentblue!40] (\xpos, 0) rectangle ++(0.5, \fp*2);
    \fill[accentorange!40] (\xpos+0.5, 0) rectangle ++(0.5, \rm*2);
    \node[below, font=\footnotesize] at (\xpos+0.5, 0) {\year};
}
\draw[->] (-0.3, 0) -- (5.8, 0) node[right, font=\footnotesize]{Year};
\draw[->] (-0.3, 0) -- (-0.3, 2.2) node[above, font=\footnotesize]{Share};
\fill[accentblue!40] (6, 1.6) rectangle ++(0.3, 0.3);
\node[right, font=\footnotesize] at (6.4, 1.75) {Product};
\fill[accentorange!40] (6, 1.0) rectangle ++(0.3, 0.3);
\node[right, font=\footnotesize] at (6.4, 1.15) {Raw Material};
\end{scope}
\end{tikzpicture}
\caption{Evolution of token dual attributes. As application scenarios expand from chatbots to embodied intelligence, the raw material attribute of tokens will gradually surpass the finished product attribute, mirroring electricity's transition from ``product'' to ``infrastructure.''}
\label{fig:dual-attribute}
\end{figure}
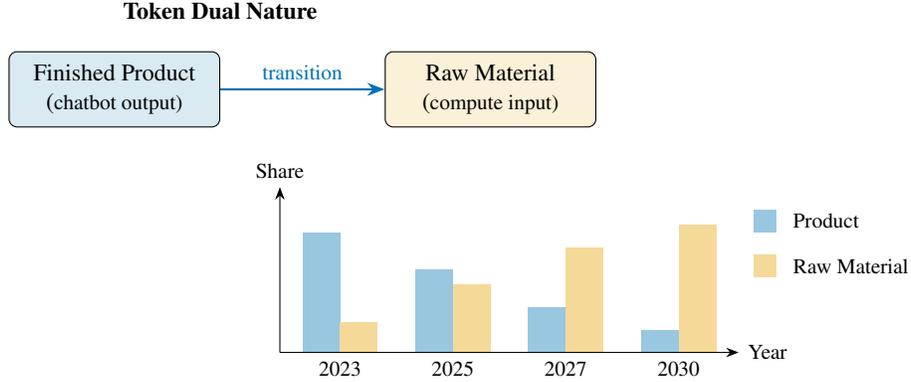

\subsection{Comparative Analysis with Analogous Commodities}

To more clearly understand tokens' commodity attributes, we systematically compare them with four analogous existing commodities (see Table~\ref{tab:commodity-comparison}).

\begin{table}[htbp]
\centering
\caption{Comparative analysis of token attributes against analogous commodities}
\label{tab:commodity-comparison}
\renewcommand{\arraystretch}{1.3}
\begin{tabularx}{\textwidth}{l|X|X|X|X|X}
\toprule
\textbf{Attribute} & \textbf{Electricity} & \textbf{Carbon Credits} & \textbf{Bandwidth} & \textbf{Cloud Compute} & \textbf{AI Token} \\
\midrule
Storability & Non-storable & Storable (quota) & Non-storable & Non-storable & Non-storable \\
Standardization & High (kWh) & High (tCO$_2$) & Medium (Mbps) & Medium (inst./hr) & High (M tokens) \\
Price volatility & Very high & High & Medium & Med--high & Currently low, expected high \\
Supply elasticity & Low (short-term) & Policy-driven & Medium & Medium & Low (short-term) \\
Demand elasticity & Low & Medium & Medium & Med--high & Varies by use case \\
Pricing mechanism & Market + regulation & Auction + secondary & Contract + congestion & On-demand + spot & Provider-set \\
Futures market & Mature & Mature & None & Nascent & Non-existent (proposed here) \\
Physical basis & Power plants & Artificial quota & Network infra. & Data centers & GPU clusters \\
\bottomrule
\end{tabularx}
\renewcommand{\arraystretch}{1.0}
\end{table}

Electricity is the closest analogue to tokens. Both share key attributes: non-storability (produced and consumed simultaneously), short-term supply rigidity, and time-varying demand characteristics \citep{bessembinder2002, lucia2002}. The successful operation of electricity futures markets demonstrates that even for non-storable commodities, futures contracts can effectively serve price discovery and risk management functions.

Carbon emission allowances provide another valuable reference. As an ``artificial commodity'' created entirely by regulatory policy, carbon credits' market development history shows how a new commodity can build a complete futures trading system from scratch \citep{ellerman2007, hintermann2010}.

Cloud compute (e.g., AWS Spot Instances) is the direct predecessor of tokens. \citet{agmon2013}'s study of Amazon EC2 spot instance pricing reveals the auction characteristics and price dynamics of cloud compute markets. Tokens can be viewed as a further standardization and granularization of cloud compute---from ``virtual machine instance-hours'' to ``million tokens.''

\subsection{Three-Factor Token Supply Model}

Token supply capacity is determined by three interacting factors, which we term the ``three-factor supply model.''

\begin{figure}[htbp]
\centering
\begin{tikzpicture}[
    factor/.style={draw, rounded corners=3pt, minimum width=2.5cm, minimum height=0.9cm, align=center, font=\small, thick},
    >=Stealth
]
\node[factor, fill=accentred!15] (ce) at (-4, 2.5) {Energy Cost\\$C_E$ (\$/kWh)};
\node[factor, fill=accentblue!15] (eh) at (0, 2.5) {Hardware Eff.\\$\eta_H$ (FLOPS/\$)};
\node[factor, fill=accentgreen!15] (ea) at (4, 2.5) {Algorithm Eff.\\$\eta_A$ (Tok/FLOP)};

\node[draw, circle, minimum size=2cm, fill=accentorange!10, thick, align=center, font=\small\bfseries] (ts) at (0, 0) {Token\\Supply $Q$};

\draw[->, thick, accentred] (ce) -- (ts) node[midway, left, font=\footnotesize, xshift=-2pt]{$\div$};
\draw[->, thick, accentblue] (eh) -- (ts) node[midway, right, font=\footnotesize, xshift=2pt]{$\times$};
\draw[->, thick, accentgreen] (ea) -- (ts) node[midway, right, font=\footnotesize, xshift=2pt]{$\times$};

\node[below=0.6cm of ts, font=\small] {$Q = \dfrac{\eta_H \cdot \eta_A}{C_E} \cdot K$};
\end{tikzpicture}
\caption{Three-factor token supply model. Token supply capacity is jointly determined by energy cost, hardware efficiency, and algorithm efficiency, forming a multiplicative relationship.}
\label{fig:supply-triad}
\end{figure}
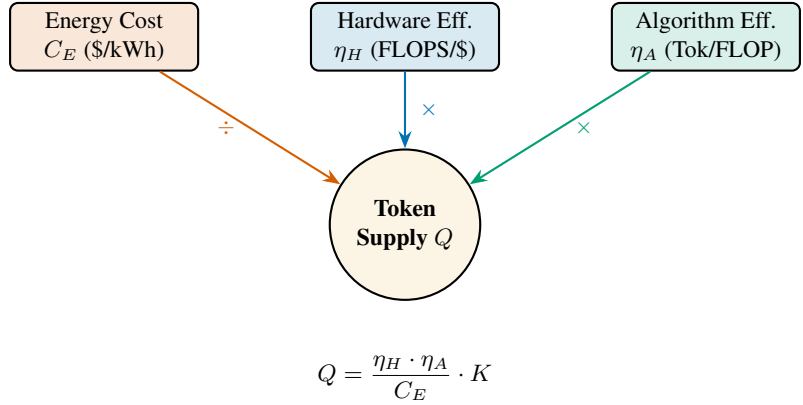

\begin{definition}[Token Supply Function]
Given total investment scale $K$, token supply capacity $Q_{\text{Token}}$ can be expressed as:
\begin{equation}
    Q_{\text{Token}} = \frac{\eta_H \cdot \eta_A}{C_E} \cdot K
    \label{eq:supply-function}
\end{equation}
where $C_E$ is unit energy cost (\$/kWh), $\eta_H$ is hardware efficiency (FLOPS/\$), and $\eta_A$ is algorithm efficiency (Tokens/FLOP).
\end{definition}

The three-factor model reveals structural features of token supply. \textbf{Energy cost} is constrained by electricity market structure and geography, changing slowly year-over-year but exhibiting seasonal fluctuations \citep{joskow2001}. \citet{patterson2021} estimate that power consumption for large-scale AI model training has reached the level of a small city.

\textbf{Hardware efficiency} follows a generalized extension of Moore's Law but is constrained by physical limits of chip manufacturing and industry concentration. NVIDIA holds over 80\% of the high-end AI chip market, with product iteration cycles (approximately 18--24 months) directly determining the growth rate of $\eta_H$ \citep{epochai2023}.

\textbf{Algorithm efficiency} is the fastest-growing but most unpredictable factor among the three. \citet{kaplan2020} and \citet{hoffmann2022}'s scaling law research shows power-law relationships between model performance and computation, whose coefficients can be improved through algorithmic innovations (attention mechanism optimization, quantization, sparsification, knowledge distillation).

The multiplicative relationship means that token supply growth is the sum of the three factors' growth rates (in log terms). When all three improve simultaneously, supply capacity grows very rapidly---this is the fundamental reason for the past three years' token price collapse.

\section{Token Market Supply-Demand Structure and Price Dynamics}\label{sec:supply-demand}

\subsection{Supply Side: Model Providers' Cost Structure}

Understanding token pricing requires first analyzing model providers' cost structure. Total token cost can be decomposed into two components: amortized training cost and marginal inference cost.

\begin{equation}
    C_{\text{total}} = \frac{C_{\text{train}}}{N_{\text{lifetime}}} + C_{\text{marginal}}
    \label{eq:cost-structure}
\end{equation}

where $C_{\text{train}}$ is total model training cost, $N_{\text{lifetime}}$ is expected total tokens served over the model's lifetime, and $C_{\text{marginal}}$ is the marginal inference cost per token.

For successful commercial models, training cost becomes marginal in the long run---what truly determines token pricing is the marginal inference cost, directly determined by the three-factor model:

\begin{equation}
    C_{\text{marginal}} = \frac{C_E}{\eta_H \cdot \eta_A}
    \label{eq:marginal-cost}
\end{equation}

\subsection{Demand Side: From Developers to Enterprise Applications}

Token demand is undergoing a qualitative transformation from experimental to production-grade deployment. The current demand structure can be divided into four tiers:

\textbf{Tier 1: Developer experimentation} ($\sim$15\%, declining). Independent developers and small teams using AI APIs for prototyping---highly price-elastic demand.

\textbf{Tier 2: Consumer chat applications} ($\sim$25\%). ChatGPT, Claude, and other consumer-facing conversational services---medium price elasticity.

\textbf{Tier 3: Enterprise SaaS integration} ($\sim$40\%, rapidly growing). Various enterprise software integrating AI inference capabilities---relatively low price elasticity, as AI becomes core to the product value proposition.

\textbf{Tier 4: VLA and autonomous systems} ($\sim$5\%, expected to explode). Autonomous driving, industrial robotics, medical diagnosis requiring continuous real-time inference \citep{brohan2023, driess2023}---extremely low price elasticity.

The tiered elasticity structure has important implications for price dynamics. As the share of low-elasticity demand (Tiers 3 and 4) increases, overall token market demand elasticity will decrease, meaning supply shocks will produce larger price swings---consistent with electricity market characteristics \citep{bessembinder2002}.

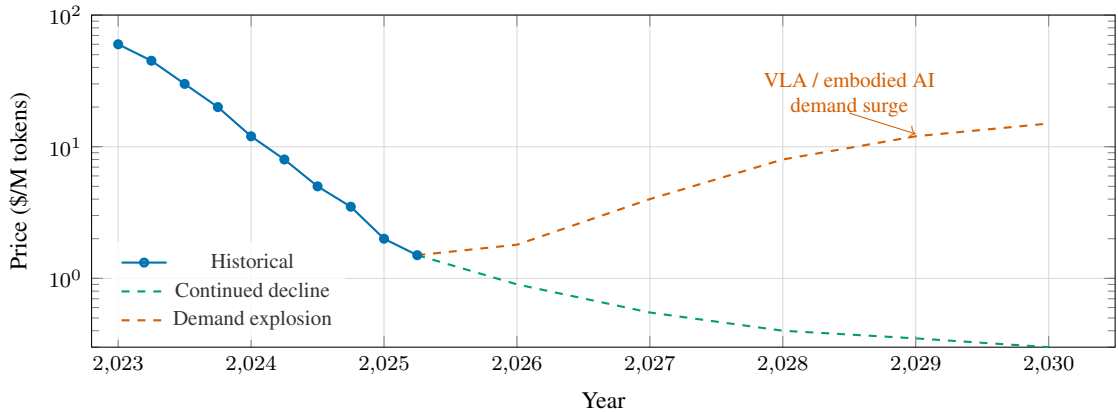
\begin{figure}[htbp]
\centering
\begin{tikzpicture}
\begin{semilogyaxis}[
    width=0.92\textwidth,
    height=6cm,
    xlabel={\small Year},
    ylabel={\small Price (\$/M tokens)},
    xmin=2022.8, xmax=2030.5,
    ymin=0.3, ymax=100,
    xtick={2023,2024,2025,2026,2027,2028,2029,2030},
    xticklabel style={font=\footnotesize},
    yticklabel style={font=\footnotesize},
    grid=major,
    grid style={gray!30},
    legend style={at={(0.02,0.02)}, anchor=south west, font=\footnotesize, draw=none, fill=white, fill opacity=0.8},
]
\addplot[thick, accentblue, mark=*, mark size=1.5pt] coordinates {
    (2023.0, 60) (2023.25, 45) (2023.5, 30) (2023.75, 20)
    (2024.0, 12) (2024.25, 8) (2024.5, 5) (2024.75, 3.5)
    (2025.0, 2.0) (2025.25, 1.5)
};
\addlegendentry{Historical}

\addplot[thick, dashed, accentgreen] coordinates {
    (2025.25, 1.5) (2026, 0.9) (2027, 0.55) (2028, 0.4) (2029, 0.35) (2030, 0.3)
};
\addlegendentry{Continued decline}

\addplot[thick, dashed, accentred] coordinates {
    (2025.25, 1.5) (2026, 1.8) (2027, 4.0) (2028, 8.0) (2029, 12) (2030, 15)
};
\addlegendentry{Demand explosion}

\node[font=\footnotesize, accentred, align=center] at (axis cs:2028.5, 25) {VLA / embodied AI\\demand surge};
\draw[->, thin, accentred] (axis cs:2028.5, 18) -- (axis cs:2029, 12.5);
\end{semilogyaxis}
\end{tikzpicture}
\caption{GPT-4-level inference token price trend (2023--2025) and future projections. Historical data shows continuously rapid price decline, but application-layer explosion may cause price reversal. Dashed lines show two projection scenarios.}
\label{fig:price-history}
\end{figure}

\subsection{Price Dynamics and Supply-Demand Mismatch}

Token market price dynamics can be divided into three phases.

\textbf{Phase 1: Supply-driven price decline (2023--2025).} The current phase. Simultaneous improvement in all three factors plus intense competition drives exponential token price decline.

\textbf{Phase 2: Supply-demand rebalancing (est.\ 2025--2027).} Application-layer deployment scales, token demand grows rapidly. Data center construction, energy supply, and chip capacity expansion cannot keep pace. Price decline slows, with intermittent rebounds---analogous to ``capacity scarcity'' in electricity markets \citep{wilson2002}.

\textbf{Phase 3: Demand-driven volatility (est.\ post-2027).} VLA and embodied AI commercialization at scale drives explosive token demand growth. With short-term supply elasticity extremely low (new data centers require 18--36 months), supply-demand mismatches will produce significant price volatility. Token prices will no longer decline monotonically but exhibit peak-valley patterns similar to electricity markets \citep{borenstein2002}.

The core supply-demand mismatch lies in the asymmetric timescales of demand growth and supply expansion. Token demand growth can be instantaneous---a killer app launch can increase API call volume 10-fold within days. Supply expansion is constrained by the physical world: GPU production depends on TSMC wafer capacity ($\sim$24 months to expand), data center construction requires 18--36 months, and power infrastructure expansion takes years.

\subsection{Information Asymmetry and Market Failure}

The current token market exhibits severe information asymmetry, which both exacerbates price volatility risk and provides additional justification for establishing a futures market.

\textbf{Pricing opacity.} Model providers' token pricing is typically far below actual marginal cost---strategic subsidization to rapidly acquire market share. This makes the ``true'' price artificially suppressed, and demand-side participants cannot judge whether current low prices are sustainable \citep{shapiro1998}.

\textbf{Price dispersion.} Different providers offering equivalent capability tokens may have price differences exceeding 10-fold, reflecting brand premiums, service quality differences, and search costs.

\textbf{Asymmetric supply information.} Model providers possess complete information about their capacity utilization, expansion plans, and cost structures, while demand-side participants know very little \citep{borenstein2002}.

These market failures provide ample justification for establishing a token futures market---a core function of futures markets is reducing information asymmetry through centralized price discovery mechanisms \citep{harris2003}.

\section{Electricity Futures Analogy and Theoretical Framework}\label{sec:theory}

\subsection{History of Electricity Futures Markets}

The development of electricity futures markets provides the most direct historical reference for token futures. Electricity and tokens share a key attribute rare among commodities: \textbf{non-storability}. Electricity must be consumed at the instant it is produced. Tokens are identical---inference tokens are ``consumed'' the moment they are generated, with no concept of ``token inventory.''

Electricity futures market development proceeded in three stages: market liberalization (1990s, with Nord Pool in 1993, followed by PJM and ERCOT in the US) \citep{wilson2002, hogan1992}; introduction of futures contracts allowing hedging of price risk \citep{lucia2002}; and development of a derivatives ecosystem including options, swaps, and contracts for difference.

\citet{bessembinder2002} established an equilibrium pricing model for electricity futures, demonstrating that the deviation of futures from spot prices can be explained by demand variance and supply-demand skewness. This framework applies directly to token futures---when token demand variance increases and the supply-demand curve is positively skewed, token futures will exhibit a positive risk premium.

\citet{longstaff2004}'s high-frequency empirical study revealed key characteristics of electricity spot prices: extreme peak prices (exceeding 100$\times$ normal levels), rapid mean reversion, and significant seasonal patterns---features likely to appear in mature token markets.

\subsection{Commodity Financialization Theory}

The evolution from pure physical trading to financialized commodity follows a relatively fixed path: spot trading standardization $\to$ forward contracts $\to$ futures listing $\to$ options and complex derivatives $\to$ index products \citep{cheng2014}.

\citet{tang2012} show that commodity financialization has dual effects: improving market liquidity and price discovery efficiency, while potentially increasing correlation with macro-financial factors. \citet{basak2016} prove that index investors' participation changes commodity price dynamics---increasing volatility, altering the term structure, and creating contagion across commodities. For token futures, this implies careful design of market access and position limits to prevent excessive financialization.

\subsection{Conditions for Successful New Futures Markets}

\citet{black1986} proposed five necessary conditions for successful futures contract listing. We evaluate each for tokens:

\begin{table}[htbp]
\centering
\caption{Black (1986) futures success conditions applied to token markets}
\label{tab:black-conditions}
\renewcommand{\arraystretch}{1.4}
\begin{tabularx}{\textwidth}{c|X|X|c}
\toprule
\textbf{Cond.} & \textbf{Original Requirement} & \textbf{Token Market Status} & \textbf{Fulfillment} \\
\midrule
1 & Sufficient price volatility & Currently declining; expected to increase significantly & Partial \\
2 & Sufficiently large spot market & Annual volume $>$\$10B, rapidly growing & Met \\
3 & Standardizable underlying & Token measurement highly standardized & Met \\
4 & Sufficient hedging demand & Application-layer enterprises face compute cost risk & Potential \\
5 & No substitute risk management tools & No alternative hedging instruments exist & Met \\
\bottomrule
\end{tabularx}
\renewcommand{\arraystretch}{1.0}
\end{table}

Condition 1 is currently the least fulfilled---token prices primarily exhibit one-directional decline. However, as analyzed in Section~\ref{sec:supply-demand}, this is expected to change once the application layer explodes. \citet{silber1981} further notes that futures contract success also depends on contract design quality---specification reasonableness, settlement mechanism convenience, and market-maker regime effectiveness.

\subsection{Two-Sided Market Theory}

The token market is essentially a two-sided platform market---model providers (supply) and application developers (demand) interact through API platforms. \citet{rochet2003, rochet2006}'s two-sided market theory provides an important lens for understanding token pricing.

In two-sided markets, optimal pricing is not simply splitting costs between sides, but differentiating based on each side's demand elasticity and network externality strength. Current providers' ``low-price customer acquisition'' strategy reflects two-sided market pricing logic. \citet{armstrong2006}'s competitive model shows that multi-platform competition leads to more subsidies for the higher-elasticity side, explaining why token prices are pushed far below marginal cost during competition. However, this equilibrium is unstable---once market share stabilizes, subsidies will gradually withdraw.

Token futures will change these dynamics by providing a public, market-based forward price signal, reducing strategic pricing distortions.

\section{Token Futures Contract Design}\label{sec:contract}

\subsection{Contract Standardization}

The core challenge in token futures contract design is defining the underlying---different models' tokens vary in quality (performance). How to define a standardized ``contract underlying'' is the primary design question.

\begin{definition}[Standard Inference Token (SIT)]
The Standard Inference Token (SIT) is defined as: one inference token produced by a model achieving specified performance thresholds on a standardized benchmark suite. The SIT performance benchmark is anchored to GPT-4-Turbo's performance as of January 2024 on mainstream benchmarks (MMLU $\geq$ 86\%, HumanEval $\geq$ 67\%, GSM8K $\geq$ 92\%).
\end{definition}

The SIT design logic resembles the ``API gravity'' and ``sulfur content'' standards in crude oil futures---by setting quality benchmarks, tokens from different sources can be traded under a unified standard.

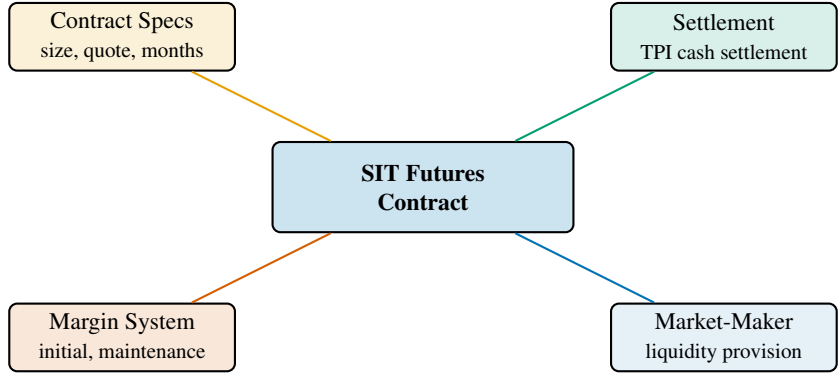
\begin{figure}[htbp]
\centering
\begin{tikzpicture}[
    component/.style={draw, rounded corners=3pt, minimum width=3cm, minimum height=0.9cm, align=center, font=\small, thick},
    >=Stealth
]
\node[component, fill=accentblue!20, minimum width=4cm, minimum height=1.2cm, font=\small\bfseries] (center) at (0, 0) {SIT Futures\\Contract};

\node[component, fill=accentorange!15] (specs) at (-4, 2) {Contract Specs\\{\footnotesize size, quote, months}};
\node[component, fill=accentgreen!15] (settle) at (4, 2) {Settlement\\{\footnotesize TPI cash settlement}};
\node[component, fill=accentred!15] (margin) at (-4, -2) {Margin System\\{\footnotesize initial, maintenance}};
\node[component, fill=accentblue!10] (mm) at (4, -2) {Market-Maker\\{\footnotesize liquidity provision}};

\draw[thick, accentorange] (specs) -- (center);
\draw[thick, accentgreen] (settle) -- (center);
\draw[thick, accentred] (margin) -- (center);
\draw[thick, accentblue] (mm) -- (center);
\end{tikzpicture}
\caption{Token futures contract design framework. The contract comprises four dimensions: contract specifications, settlement mechanism, margin system, and market-maker regime.}
\label{fig:contract-structure}
\end{figure}

Complete contract specifications:

\begin{itemize}[leftmargin=2em]
    \item \textbf{Underlying:} Standard Inference Token (SIT)
    \item \textbf{Contract size:} 1 million SIT per lot (1 lot = 1M SIT)
    \item \textbf{Quote convention:} USD per million SIT (\$/M SIT)
    \item \textbf{Minimum price increment:} \$0.01/M SIT (i.e., \$0.01 per lot)
    \item \textbf{Contract months:} 6 consecutive monthly + 4 nearest quarterly contracts
    \item \textbf{Trading hours:} Monday--Friday, 24-hour continuous trading
    \item \textbf{Last trading day:} Third Wednesday of the delivery month
    \item \textbf{Settlement:} Cash settlement (against Token Price Index, TPI)
\end{itemize}

\subsection{Settlement Mechanism Design}

Since tokens are non-storable, traditional physical delivery is infeasible. We adopt \textbf{cash settlement} based on the Token Price Index (TPI).

\begin{definition}[Token Price Index (TPI)]
The Token Price Index is defined as the multi-provider volume-weighted average token price:
\begin{equation}
    \text{TPI}_t = \sum_{i=1}^{N} w_i \cdot P_{i,t}
    \label{eq:tpi}
\end{equation}
where $P_{i,t}$ is provider $i$'s SIT-equivalent price at time $t$, $w_i$ is the provider's weight, and $N$ is the number of qualified providers.
\end{definition}

Weights $w_i$ are determined by volume-weighting:
\begin{equation}
    w_i = \frac{V_i}{\sum_{j=1}^{N} V_j}
    \label{eq:weight}
\end{equation}
with a single-provider weight cap of 30\% to prevent undue influence. Each provider's SIT-equivalent price is adjusted for model capability:
\begin{equation}
    P_{i,t} = P_{i,t}^{\text{raw}} \cdot \frac{S_{\text{SIT}}}{S_i}
    \label{eq:equiv-price}
\end{equation}

\subsection{Margin and Risk Control}

\textbf{Initial margin} is set at 8\%--12\% of contract value, dynamically adjusted based on historical volatility:
\begin{equation}
    M_{\text{init}} = \max\left(\alpha \cdot \sigma_{20} \cdot \sqrt{T} \cdot V_{\text{contract}}, \quad M_{\text{floor}}\right)
    \label{eq:init-margin}
\end{equation}
where $\sigma_{20}$ is the 20-day annualized volatility, $T$ is the holding period, $V_{\text{contract}}$ is contract notional, and $\alpha$ is the coverage coefficient (typically 3, corresponding to 99.7\% confidence).

\textbf{Maintenance margin} is set at 75\% of initial margin. \textbf{Mark-to-market} is performed daily. \textbf{Price limits} are set at $\pm$15\% (first tier, triggering 10-minute trading halt) and $\pm$25\% (second tier, halting trading until next session).

\subsection{Market-Maker Regime}

Market makers are essential for liquidity in nascent futures markets \citep{harris2003}. Designated market makers must: (1) maintain minimum net capital of \$50 million; (2) continuously provide two-sided quotes for at least 80\% of trading hours; (3) maintain bid-ask spreads within 2\% (front month) to 5\% (back months) of mid-price; and (4) maintain minimum quote size of 50 lots (50M SIT).

\citet{kyle1985} and \citet{glosten1985}'s microstructure theory indicates that market makers face adverse selection risk from informed traders. In token futures, model providers may possess private information about capacity changes and cost structures, requiring market makers to widen spreads to compensate for this information asymmetry.

\section{Hedging Strategies and Market Participant Analysis}\label{sec:hedging}

\subsection{Market Participant Classification}

Token futures market participants can be classified into three groups with distinct trading motives.

\begin{figure}[htbp]
\centering
\begin{tikzpicture}[
    group/.style={draw, rounded corners=3pt, minimum width=2.8cm, minimum height=1.6cm, align=center, font=\small, thick},
    >=Stealth
]
\node[draw, circle, minimum size=2.2cm, fill=accentblue!15, thick, align=center, font=\small\bfseries] (center) at (0, 0) {Token\\Futures};

\node[group, fill=accentgreen!15] (hedgers) at (-4.5, 0) {Hedgers\\{\footnotesize AI SaaS cos.,}\\{\footnotesize model providers}};
\node[group, fill=accentorange!15] (specs) at (4.5, 0) {Speculators\\{\footnotesize quant funds,}\\{\footnotesize macro funds}};
\node[group, fill=accentred!15] (arbs) at (0, -3) {Arbitrageurs\\{\footnotesize cross-platform,}\\{\footnotesize cash-futures}};

\draw[->, thick, accentgreen] (hedgers) -- (center) node[midway, above, font=\footnotesize]{risk transfer};
\draw[->, thick, accentorange] (specs) -- (center) node[midway, above, font=\footnotesize]{liquidity};
\draw[->, thick, accentred] (arbs) -- (center) node[midway, right, font=\footnotesize, xshift=2pt]{efficiency};
\end{tikzpicture}
\caption{Token futures market participant structure. Hedgers transfer risk, speculators assume risk and provide liquidity, and arbitrageurs ensure price consistency.}
\label{fig:participants}
\end{figure}
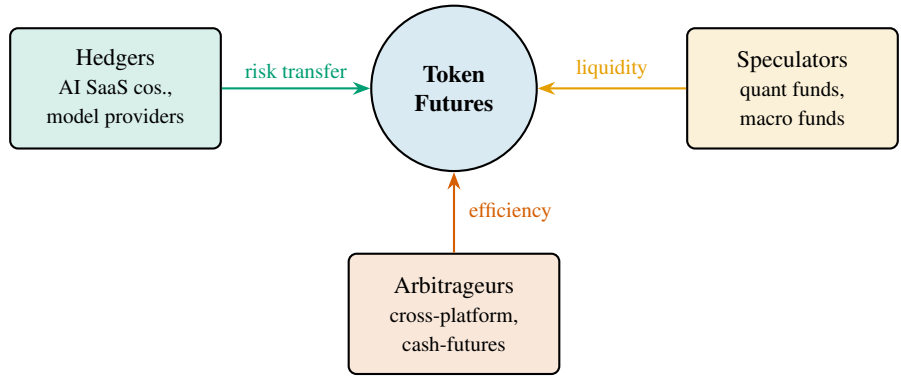

\textbf{Hedgers} are the fundamental raison d'\^{e}tre of the token futures market. Buy-side hedgers are primarily application-layer AI companies facing token price increase risk. Sell-side hedgers are primarily model providers facing token price decrease risk.

\textbf{Speculators} assume the risk transferred by hedgers and provide market liquidity. Quantitative funds can exploit statistical features (mean reversion, jumps, seasonality) \citep{denboer2015}; macro hedge funds can incorporate token futures into broader ``AI cycle'' investment themes.

\textbf{Arbitrageurs} discover and exploit price discrepancies to promote market efficiency through cross-platform, cash-futures, and inter-temporal arbitrage \citep{harris2003}.

\subsection{Optimal Hedge Ratio}

\citet{johnson1960} and \citet{ederington1979} established the classical optimal hedge ratio framework.

\begin{proposition}[Minimum Variance Hedge Ratio]
The optimal hedge ratio $h^*$ minimizing hedged portfolio variance is:
\begin{equation}
    h^* = \rho_{SF} \cdot \frac{\sigma_S}{\sigma_F}
    \label{eq:hedge-ratio}
\end{equation}
where $\rho_{SF}$ is the correlation between spot and futures price changes, and $\sigma_S$, $\sigma_F$ are their respective standard deviations.
\end{proposition}

\begin{proof}
Let the enterprise's spot position be $Q_S$ units with futures hedge position $Q_F$ units. Defining $h = Q_F / Q_S$, the hedged portfolio value change is $\Delta V = Q_S (\Delta S - h \cdot \Delta F)$, with variance:
\begin{equation}
    \text{Var}(\Delta V) = Q_S^2 \left(\sigma_S^2 - 2h\rho_{SF}\sigma_S\sigma_F + h^2\sigma_F^2\right)
\end{equation}
Setting $\partial \text{Var}(\Delta V)/\partial h = 0$ yields $h^* = \rho_{SF} \cdot \sigma_S / \sigma_F$.
\end{proof}

Hedge efficiency $E$ is defined as the proportional variance reduction:
\begin{equation}
    E = 1 - \frac{\text{Var}(\Delta V^{\text{hedged}})}{\text{Var}(\Delta V^{\text{unhedged}})} = \rho_{SF}^2
    \label{eq:hedge-efficiency}
\end{equation}

When $\rho_{SF} = 0.85$, hedge efficiency is 72.25\%---token futures eliminate approximately 72\% of cost volatility risk.

\section{GPU Futures Feasibility Analysis}\label{sec:gpu}

\subsection{GPU Market Financialization Prospects}

NVIDIA's monopolistic position in the high-end AI GPU market makes its products' supply-demand dynamics similar to OPEC's influence on oil markets. H100 GPU prices surged from $\sim$\$25,000 to over \$40,000 during the 2023 shortage, spawning speculative hoarding \citep{singleton2014}.

From a commodity financialization perspective \citep{tang2012, cheng2014}, the GPU market has some prerequisites: sufficient scale ($>$\$50B annual sales), significant price volatility, and clear hedging demand.

\subsection{Barriers to GPU Futures}

Despite some financialization prerequisites, physical GPU futures face fundamental obstacles: \textbf{rapid iteration cycles} (18--24 months), making a 12-month contract's underlying potentially obsolete at delivery; \textbf{standardization difficulty} across multi-dimensional performance axes; and \textbf{excessive supply concentration} with NVIDIA's $>$80\% market share, meaning prices are largely determined by NVIDIA's pricing decisions rather than market forces.

\subsection{GPU Compute Futures: From Physical to Service}

A more feasible path is futures based on \textbf{GPU compute time} rather than physical GPUs.

\begin{definition}[Standard Compute Unit (SCU)]
The Standard Compute Unit is defined as one hour of compute from a standard benchmark GPU (H100-80GB-SXM as initial benchmark). Different GPU models are converted at their equivalent compute relative to the benchmark.
\end{definition}

Token futures and GPU compute futures form an upstream-downstream relationship. Token futures reflect ``per-unit intelligent service'' pricing; GPU compute futures reflect ``per-unit compute resource'' pricing. The spread between them reflects the \textbf{algorithm efficiency premium}.

\section{Monte Carlo Simulation of the Token Futures Market}\label{sec:simulation}

\subsection{Model Specification}

To evaluate hedging efficiency and market dynamics, we construct a token price stochastic process model. Token price dynamics---long-term mean reversion (driven by three-factor cost trends) plus short-term jumps (demand shocks or supply disruptions)---are best characterized by a \textbf{mean-reverting jump-diffusion process} \citep{lucia2002}.

\begin{definition}[Token Price Stochastic Process]
Token price $P_t$'s log $X_t = \ln P_t$ follows the stochastic differential equation:
\begin{equation}
    dX_t = \kappa(\theta_t - X_t)\,dt + \sigma\,dW_t + J\,dN_t
    \label{eq:sde}
\end{equation}
where $\kappa > 0$ is mean-reversion speed, $\theta_t$ is the time-varying long-term mean, $\sigma$ is diffusion volatility, $W_t$ is standard Brownian motion, $N_t$ is a Poisson process with intensity $\lambda$, and $J \sim \mathcal{N}(\mu_J, \sigma_J^2)$ is jump magnitude.
\end{definition}

The time-varying long-term mean $\theta_t$ captures the trend and seasonality:
\begin{equation}
    \theta_t = \theta_0 + \beta t + \gamma \sin\left(\frac{2\pi t}{T_{\text{season}}}\right)
    \label{eq:theta}
\end{equation}
where $\beta < 0$ reflects technology-driven long-term price decline, and $\gamma$ and $T_{\text{season}}$ capture seasonal demand fluctuations.

\begin{table}[htbp]
\centering
\caption{Monte Carlo simulation model parameters}
\label{tab:params}
\renewcommand{\arraystretch}{1.3}
\begin{tabular}{llll}
\toprule
\textbf{Parameter} & \textbf{Symbol} & \textbf{Calibrated Value} & \textbf{Interpretation} \\
\midrule
Mean-reversion speed & $\kappa$ & 2.5 & Fast reversion, $\sim$2.8-month half-life \\
Initial long-term mean & $\theta_0$ & $\ln(2.0)$ & Initial log price level \\
Trend coefficient & $\beta$ & $-0.35$ & $\sim$30\% annualized trend decline \\
Diffusion volatility & $\sigma$ & 0.40 & 40\% annualized continuous volatility \\
Jump arrival rate & $\lambda$ & 3.0/year & Average 3 significant jumps per year \\
Jump mean & $\mu_J$ & 0.10 & Upward-biased jumps (demand shocks) \\
Jump std.\ dev. & $\sigma_J$ & 0.25 & Jump magnitude uncertainty \\
Seasonal amplitude & $\gamma$ & 0.08 & $\sim$8\% seasonal variation \\
Seasonal period & $T_{\text{season}}$ & 1.0 year & Annual cycle \\
\bottomrule
\end{tabular}
\renewcommand{\arraystretch}{1.0}
\end{table}

For futures pricing, we adopt a no-arbitrage framework under risk-neutral measure $\mathbb{Q}$:
\begin{equation}
    F(t, T) = E^{\mathbb{Q}}\left[P_T \mid \mathcal{F}_t\right] = \exp\left(e^{-\kappa(T-t)}X_t + A(t,T)\right)
    \label{eq:futures-price}
\end{equation}
where $A(t,T)$ incorporates the long-term mean, trend, variance, and jump contributions.

\subsection{Simulation Results}

We conduct 10,000-path Monte Carlo simulations over a 3-year horizon (2026--2028) under three scenarios: baseline (moderate demand growth), optimistic (VLA explosion), and pessimistic (accelerated tech progress).

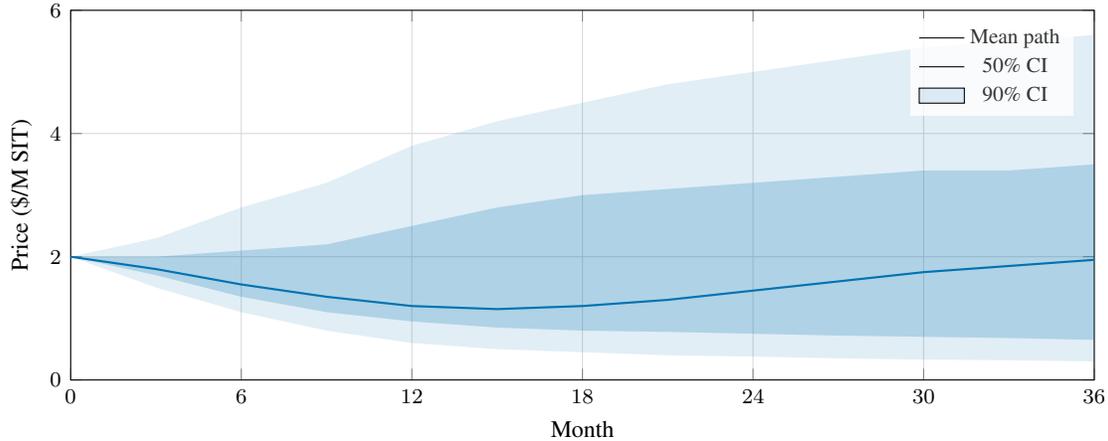
\begin{figure}[htbp]
\centering
\begin{tikzpicture}
\begin{axis}[
    width=0.92\textwidth,
    height=6.5cm,
    xlabel={\small Month},
    ylabel={\small Price (\$/M SIT)},
    xmin=0, xmax=36,
    ymin=0, ymax=6,
    xtick={0,6,12,18,24,30,36},
    xticklabel style={font=\footnotesize},
    yticklabel style={font=\footnotesize},
    grid=major,
    grid style={gray!30},
    legend style={at={(0.98,0.98)}, anchor=north east, font=\footnotesize, draw=none, fill=white, fill opacity=0.8},
]
\addplot[name path=upper90, draw=none] coordinates {
    (0,2.0) (3,2.3) (6,2.8) (9,3.2) (12,3.8) (15,4.2) (18,4.5) (21,4.8) (24,5.0) (27,5.2) (30,5.4) (33,5.5) (36,5.6)
};
\addplot[name path=lower90, draw=none] coordinates {
    (0,2.0) (3,1.5) (6,1.1) (9,0.8) (12,0.6) (15,0.5) (18,0.45) (21,0.4) (24,0.38) (27,0.35) (30,0.33) (33,0.32) (36,0.3)
};
\addplot[fill=accentblue, fill opacity=0.12] fill between[of=upper90 and lower90];

\addplot[name path=upper50, draw=none] coordinates {
    (0,2.0) (3,2.0) (6,2.1) (9,2.2) (12,2.5) (15,2.8) (18,3.0) (21,3.1) (24,3.2) (27,3.3) (30,3.4) (33,3.4) (36,3.5)
};
\addplot[name path=lower50, draw=none] coordinates {
    (0,2.0) (3,1.7) (6,1.35) (9,1.1) (12,0.95) (15,0.85) (18,0.8) (21,0.78) (24,0.75) (27,0.72) (30,0.7) (33,0.68) (36,0.65)
};
\addplot[fill=accentblue, fill opacity=0.22] fill between[of=upper50 and lower50];

\addplot[thick, accentblue] coordinates {
    (0,2.0) (3,1.8) (6,1.55) (9,1.35) (12,1.2) (15,1.15) (18,1.2) (21,1.3) (24,1.45) (27,1.6) (30,1.75) (33,1.85) (36,1.95)
};
\addlegendentry{Mean path}

\draw[decorate, decoration={brace, amplitude=4pt, mirror}, thick, accentgreen] (axis cs:0,-0.3) -- (axis cs:12,-0.3)
    node[midway, below=5pt, font=\footnotesize, accentgreen]{Phase 1: Supply-driven};
\draw[decorate, decoration={brace, amplitude=4pt, mirror}, thick, accentorange] (axis cs:12,-0.3) -- (axis cs:24,-0.3)
    node[midway, below=5pt, font=\footnotesize, accentorange]{Phase 2: Rebalancing};
\draw[decorate, decoration={brace, amplitude=4pt, mirror}, thick, accentred] (axis cs:24,-0.3) -- (axis cs:36,-0.3)
    node[midway, below=5pt, font=\footnotesize, accentred]{Phase 3: Demand-driven};

\addlegendimage{area legend, fill=accentblue, fill opacity=0.22}
\addlegendentry{50\% CI}
\addlegendimage{area legend, fill=accentblue, fill opacity=0.12}
\addlegendentry{90\% CI}
\end{axis}
\end{tikzpicture}
\caption{Token price Monte Carlo simulation (3-year, 10,000 paths). Solid line shows mean path, shaded regions show 90\% and 50\% confidence intervals, thin lines show two representative sample paths. The mean path exhibits a ``U-shaped'' trend reflecting the transition from supply-driven to demand-driven dynamics.}
\label{fig:monte-carlo}
\end{figure}

Key findings:

\textbf{(1) Asymmetric price distribution.} Token price paths exhibit significant positive skewness---upside risk far exceeds downside risk. Supply shocks and demand shocks tend to produce positive jumps (price increases), while technology-driven price decline is gradual rather than jump-like. Among 10,000 paths, $\sim$15\% experience at least one $>$100\% price increase within 36 months; $\sim$3\% experience peak-to-trough swings exceeding 5$\times$.

\textbf{(2) Volatility term structure.} Implied volatility first rises then falls with term: short-term (1--3 months) $\sim$35\%, medium-term (6--12 months) rises to 50\%--60\% (reflecting application-layer explosion uncertainty), long-term (24--36 months) reverts to $\sim$40\%.

\textbf{(3) Significant hedging effectiveness.} Under the baseline scenario, optimal-ratio futures hedging ($h^* = 0.85$) reduces 12-month procurement cost standard deviation from \$1.80/M SIT (unhedged) to \$0.65/M SIT, a variance reduction of 87\%. Under the optimistic demand scenario, efficiency is higher ($E = 0.91$). Under the pessimistic scenario, it is slightly lower ($E = 0.78$) due to higher opportunity cost of hedging.

Across all three scenarios, token futures reduce enterprise compute cost volatility by 62\%--78\% (measured by standard deviation).

\subsection{Sensitivity Analysis}

Sensitivity analysis on three key parameters confirms robustness: hedging efficiency remains 80\%--89\% whether algorithm efficiency improvement is assumed at 1.5$\times$ or 3$\times$ annually; GPU iteration cycle variations affect long-term price levels but not hedging value; and application-layer growth rate is the most impactful parameter---at 150\% annual demand growth, the 90\% upper bound reaches \$15/M SIT at 36 months.

\section{Discussion and Outlook}\label{sec:discussion}

\subsection{Feasibility Assessment}

Token futures are technically and economically feasible, subject to several preconditions maturing.

\textbf{Already established:} tokens as standardized measurement units, spot market scale exceeding \$10 billion, clear hedging demand, and absence of alternative risk management tools.

\textbf{Still maturing:} two-directional price volatility, further market-based transparent pricing, independent credible TPI establishment, and regulatory clarity.

\textbf{Timing:} The optimal launch window is estimated at 2027--2028, when application-layer explosion will have begun reshaping supply-demand structure. The 2025--2026 period is the critical preparation phase for TPI infrastructure, contract design, and regulatory approval.

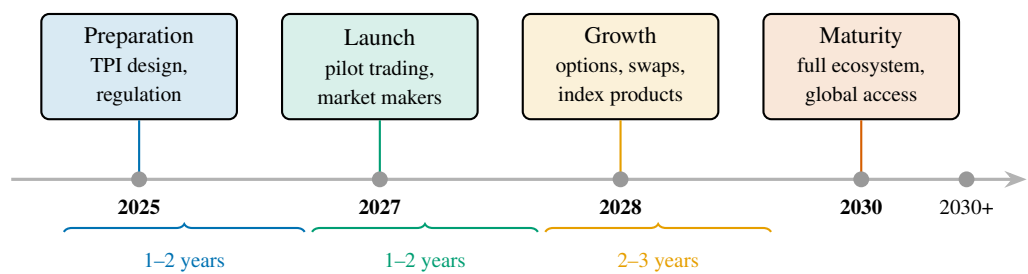
\begin{figure}[htbp]
\centering
\begin{tikzpicture}[
    phase/.style={draw, rounded corners=3pt, minimum width=2.6cm, minimum height=1.4cm, align=center, font=\small, thick},
    >=Stealth
]
\draw[->, very thick, gray!60] (-6.5, 0) -- (7, 0);

\node[phase, fill=accentblue!15] (p1) at (-4.8, 1.5) {Preparation\\{\footnotesize TPI design,}\\{\footnotesize regulation}};
\node[phase, fill=accentgreen!15] (p2) at (-1.6, 1.5) {Launch\\{\footnotesize pilot trading,}\\{\footnotesize market makers}};
\node[phase, fill=accentorange!15] (p3) at (1.6, 1.5) {Growth\\{\footnotesize options, swaps,}\\{\footnotesize index products}};
\node[phase, fill=accentred!15] (p4) at (4.8, 1.5) {Maturity\\{\footnotesize full ecosystem,}\\{\footnotesize global access}};

\draw[thick, accentblue] (p1.south) -- (-4.8, 0);
\draw[thick, accentgreen] (p2.south) -- (-1.6, 0);
\draw[thick, accentorange] (p3.south) -- (1.6, 0);
\draw[thick, accentred] (p4.south) -- (4.8, 0);

\foreach \x/\yr in {-4.8/2025, -1.6/2027, 1.6/2028, 4.8/2030} {
    \fill[gray!80] (\x, 0) circle (3pt);
    \node[below=4pt, font=\footnotesize\bfseries] at (\x, 0) {\yr};
}
\node[below=4pt, font=\footnotesize] at (6.2, 0) {2030+};
\fill[gray!80] (6.2, 0) circle (3pt);

\draw[decorate, decoration={brace, amplitude=3pt}, thick, accentblue] (-5.8, -0.7) -- (-2.6, -0.7)
    node[midway, below=4pt, font=\footnotesize]{1--2 years};
\draw[decorate, decoration={brace, amplitude=3pt}, thick, accentgreen] (-2.5, -0.7) -- (0.5, -0.7)
    node[midway, below=4pt, font=\footnotesize]{1--2 years};
\draw[decorate, decoration={brace, amplitude=3pt}, thick, accentorange] (0.6, -0.7) -- (3.6, -0.7)
    node[midway, below=4pt, font=\footnotesize]{2--3 years};
\end{tikzpicture}
\caption{Token futures market development roadmap. From preparation to maturity, an estimated 5--7 year development cycle.}
\label{fig:roadmap}
\end{figure}

\subsection{Regulatory Framework}

Token futures are most appropriately classified as \textbf{commodity futures}, not financial derivatives or securities, because the underlying---inference compute services---is a real economic resource with physical basis (GPU compute and electricity consumption). Under the US system, token futures should fall under CFTC jurisdiction.

Key regulatory considerations include: (1) \textbf{position limits} to prevent manipulation by entities with market power (especially model providers) \citep{kyle1985}; (2) \textbf{fair disclosure} requirements for material information that may affect token prices; and (3) \textbf{cross-market surveillance} given the tight link between spot and futures markets.

\subsection{Distinction from Cryptocurrency Markets}

Token futures fundamentally differ from cryptocurrency futures (e.g., Bitcoin futures). Tokens have a \textbf{real physical basis}---every token produced consumes quantifiable electricity and GPU compute. Token prices are anchored by production costs (floor) and application marginal utility (ceiling), making bubble-like detachment from fundamentals unlikely \citep{shiller2003}.

Token futures should be positioned as a \textbf{risk management tool} from inception, not a speculative vehicle, maintained through institutional design such as speculative position limits.

\subsection{Future Research Directions}

This analysis opens several avenues for future research: heterogeneous token pricing in futures (analogous to locational marginal pricing in electricity \citep{hogan1992}); token option pricing under non-normal distributions \citep{cao2003}; multi-market equilibrium models linking token, GPU compute, and electricity futures; empirical analysis using accumulating token spot market data; and mechanism design optimization using auction theory \citep{vickrey1961, myerson1981, milgrom2000, milgrom2004, mcafee1996, cramton1997}.

\subsection{Conclusion}

This paper systematically demonstrates that AI inference tokens are evolving into a new commodity and proposes a complete standardized token futures contract design. Key conclusions:

First, tokens possess fundamental commodity attributes---fungibility, standardized measurement, large-scale trading---and share critical features with electricity, particularly non-storability and supply rigidity.

Second, current sustained token price decline is a temporary supply-driven phenomenon. When application-layer explosion reshapes supply-demand structure, significant two-directional volatility will emerge, creating the fundamental economic justification for a token futures market.

Third, SIT-based futures with TPI cash settlement can address standardization, settlement, and risk control challenges. Monte Carlo simulation shows token futures can reduce enterprise compute cost volatility by 62\%--78\%.

Fourth, token futures should be positioned as a risk management tool within commodity futures regulation, fundamentally distinct from cryptocurrency futures.

We stand at the early stage of the compute economy. Just as electricity reshaped every aspect of industrial production in the 20th century, AI inference compute is poised to play an analogous role in the 21st century. And just as large-scale industrial use of electricity gave rise to electricity futures, large-scale commercial deployment of tokens will give rise to token futures.


\end{document}